\documentclass[runningheads,10pt]{llncs}

\usepackage{marginnote}
\usepackage{microtype}
\usepackage{graphicx}
\usepackage{subfigure}
\usepackage{booktabs} 
\usepackage{tikz}
\usepackage{mathrsfs}
\usepackage{marvosym}
\usetikzlibrary{arrows}
\usepackage{multirow}
\usepackage{hyperref}
\usepackage[all]{xy}
\usepackage{color}
\usepackage{dsfont}                             
\usepackage{helvet}                             
\usepackage{inconsolata}
\usepackage{tikz}                               
\usepackage{latexsym}
\usepackage{stmaryrd}
\usepackage{mathtools}
\usepackage{mdframed}   
\usepackage{thmtools}
\usepackage{tcolorbox}
\usepackage{enumitem}

\usepackage{amsmath}
\usepackage{amssymb}
\usepackage{mathtools}
\usepackage{xcolor}
\usepackage{graphicx}
\usepackage{caption}
\usepackage{cancel}
\usepackage{xspace}

\newcommand{\R}{\ensuremath{\mathbb{R}}\xspace}

\newcommand{\M}{\ensuremath{\mathcal{M}}\xspace}
\newcommand{\N}{\ensuremath{\widetilde{\mathcal{N}}}\xspace}

\newcommand{\Q}{\ensuremath{\mathcal{Q}}\xspace}
\newcommand{\x}{\ensuremath{\mathbf{x}}\xspace}

\newcommand{\y}{\ensuremath{\mathbf{y}}\xspace}

\newcommand{\X}{\ensuremath{\mathbf{X}}\xspace}
\newcommand{\I}{\ensuremath{\mathbf{I}}\xspace}

\newcommand{\bL}{\ensuremath{\mathbf{L}}\xspace}
\newcommand{\W}{\ensuremath{\mathbf{W}}\xspace}
\newcommand{\bH}{\ensuremath{\mathbf{H}}\xspace}
\newcommand{\Wm}{\ensuremath{\mathbf{U}}\xspace}
\newcommand{\tran}{\ensuremath{\mathit{\tran}}}
\newcommand{\Agg}{\textit{\Agg}}

\newcommand{\tA}{\ensuremath{\Tilde{\mathbf{A}}}\xspace}
\newcommand{\tD}{\ensuremath{\Tilde{\mathbf{D}}}\xspace}

\newcommand{\numNode}{\ensuremath{m\xspace}}

\usepackage[textsize=tiny]{todonotes}
\usepackage{rotating}
\usepackage{graphbox}
\usepackage{array}
\begin{document}
\newcolumntype{P}[1]{>{\centering\arraybackslash}p{#1}}
\newcolumntype{M}[1]{>{\centering\arraybackslash}m{#1}}
\newcommand{\zhou}[1]{\todo[color=blue!20]{Zhou: #1}}
\title{Alleviating Over-Smoothing via Aggregation over Compact Manifolds}
%
%

\makeatletter
\newcommand{\printfnsymbol}[1]{%
  \textsuperscript{\@fnsymbol{#1}}%
}
\makeatother

\author{Dongzhuoran Zhou\inst{1,2}\thanks{equal contribution}\textsuperscript{\Letter} \and Hui Yang\inst{3}\printfnsymbol{1}  \and Bo Xiong\inst{4} 
\and Yue Ma\inst{3}\and Evgeny Kharlamov\inst{1,2}}

%
\authorrunning{D. Zhou et al.}
%
\institute{Bosch Center for AI, Germany \\
\email{\{dongzhuoran.zhou, evgeny.kharlamov\}@de.bosch.com} \and
University of Oslo, Norway \and 
LISN, CNRS, Universite Paris-Saclay, France \\ \email{\{yang, ma\}@lisn.fr} \and
University of Stuttgart, Germany \\ \email{bo.xiong@ipvs.uni-stuttgart.de}}

\maketitle              
%
\begin{abstract}
Graph neural networks (GNNs) have achieved significant success in various applications. Most GNNs learn the node features with information aggregation of its neighbors and feature transformation in each layer. However, the node features become indistinguishable after many layers, leading to performance deterioration: a significant limitation known as over-smoothing. 
Past work adopted various techniques for addressing this issue, 
such as normalization and skip-connection of layer-wise output. 
After the study, we found that the information aggregations in existing work are all contracted aggregations, with the intrinsic property that features will inevitably converge to the same single point after many layers.
To this end, we propose the aggregation over
compacted manifolds method (ACM) that replaces the existing information aggregation with aggregation over compact manifolds, a special type of manifold, 
which avoids contracted aggregations.
In this work, we theoretically analyze contracted aggregation and its properties. We also provide an extensive empirical evaluation that shows ACM can effectively alleviate over-smoothing and outperforms the state-of-the-art. The code can be found in \href{https://github.com/DongzhuoranZhou/ACM.git}{https://github.com/DongzhuoranZhou/ACM.git}.
\end{abstract}
%
%
%
\section{Introduction}
\label{sec:intro}

Graph neural networks (GNNs)~\cite{zhou2020graph} are potent tools for analyzing graph-structured data, including biochemical networks~\cite{xu2018powerful}, social networks~\cite{hamilton2017inductive}, and academic networks~\cite{gao2018large}. Most GNNs employ a message-passing mechanism for learning node features, involving information aggregation from neighbors and feature transformation in each layer~\cite{hamilton2020graph}. This mechanism enables  GNNs to effectively capture detailed information in graph data.\looseness-1

\begin{figure}
\centering 
\subfigure[Contracted aggregation]{
\begin{minipage}[t]{0.3\linewidth}
\centering
\includegraphics[width=1.1in]{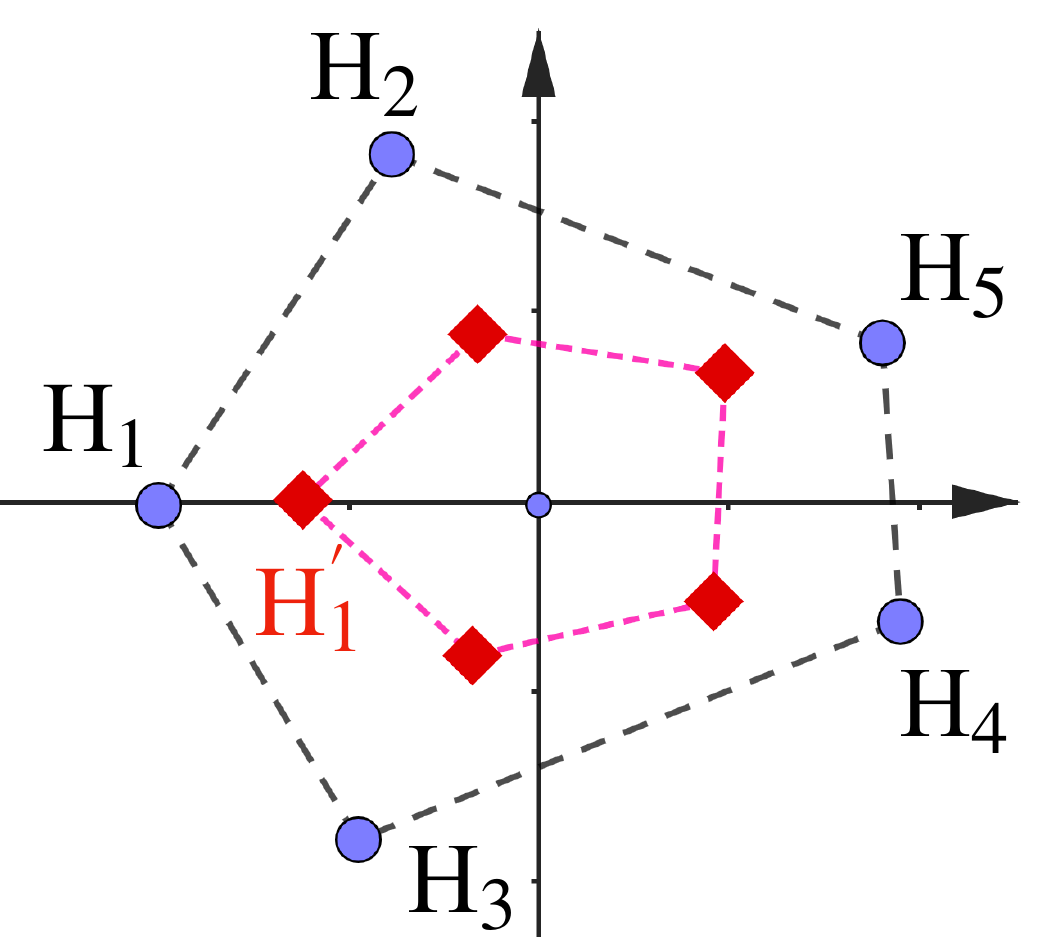}
\end{minipage}%
}%
\subfigure[Aggregation on circle]{
\begin{minipage}[t]{0.3\linewidth}
\centering
\includegraphics[width=1.1in]{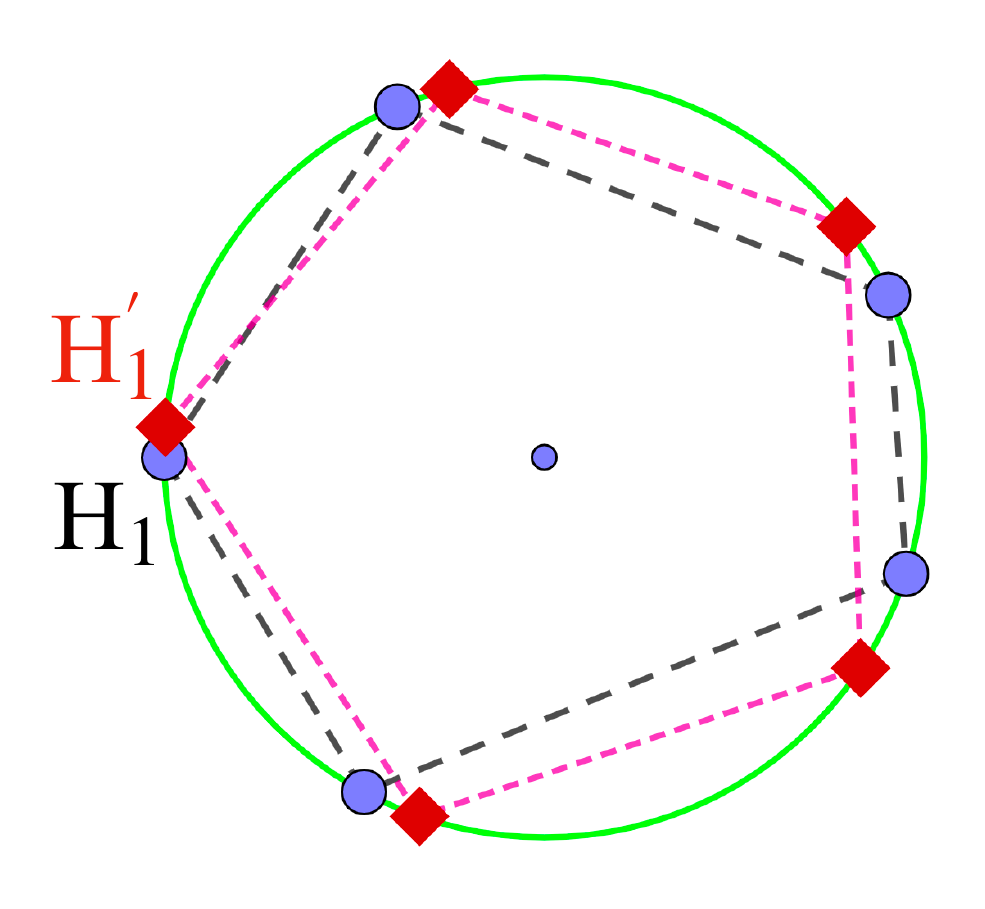}
\end{minipage}%
}%
\subfigure[Compact Manifolds]{
\begin{minipage}[t]{0.3\linewidth}
\centering
\raisebox{0.6\height}{
    \includegraphics[width= 1.1in]{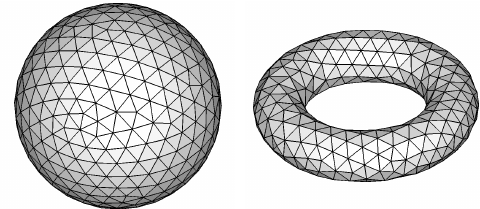}
}
\label{fig:compact}
\end{minipage}%
}%
    \caption{(a) node features with \textit{contracted aggregation}  become closer and inevitably converge after many layers. (b) Our approach avoids contracted aggregation by aggregating on compact manifolds. (c) Compact manifolds: sphere and torus.
    }
    \label{fig:intro}
\end{figure}

Though potential in various tasks,  existing GNNs 
with many layers  cause node representations to be highly  indistinguishable, thus leading 
to significant performance deterioration in downstream tasks - a phenomenon known as over-smoothing~\cite{li2018deeper}. 
Increasing efforts have been devoted to addressing  the over-smoothing issue, such as batch norm~\cite{ioffe2015batch}, pair norm~\cite{zhao2019pairnorm}, group norm~\cite{zhou2020towards}, drop edges or nodes~\cite{rong2019dropedge,chen2020measuring}, and regularize the nodes~\cite{hou2020measuring} or feature distance~\cite{jin2022feature}. Most of  existing studies focus on alleviating the over-smoothing based on enhancing the distinction in node dimension~\cite{ioffe2015batch,zhao2019pairnorm} 
or feature dimension~\cite{jin2022feature} 
 through normalization, adding regularization items~\cite{jin2022feature}, skip-connection~\cite{xu2018representation,li2019deepgcns}, etc.  


In this work, we show that past works on over-smoothing are confined to \textit{contracted aggregation} on Euclidean spaces, where features inevitably converge to a single point after multiple layers of information aggregations. The intuition of \textit{contracted aggregation} is illustrated in Fig.\ref{fig:intro}a, using the plane $\R^2$ as the embedding space and the mean of vectors as the aggregation function. The aggregation results (red points) become closer than the initial embedding (blue points), causing nodes to converge to a single point after several steps. In contrast, Fig.\ref{fig:intro}b shows that after \emph{aggregation over compact manifold} (compact manifold see Fig.\ref{fig:compact}), node features do not become closer, thus avoiding the \textit{contracted aggregation} issue\footnote{In Fig.~\ref{fig:intro}b, the unit circle is used as the embedding space. The aggregation of points is defined under \textit{polar coordinates}. For example, the aggregation of two points in the unit circle with polar coordinate  $(1, \theta), (1, \phi)$ is $(1, \frac{\theta+\phi}{2})$.}. Even after numerous information aggregation steps, the aggregation results (red points) remain distinguishable and do not converge to a single point, unlike the Euclidean case.\looseness=-1

In light of this, we perform an in-depth study on the over-smoothing in GNNs from the information aggregation perspective. We aim to bring the benefits of information aggregation on manifolds to alleviate over-smoothing. 
We develop a general theory to explain the reason for over-smoothing in Riemannian manifolds (including Euclidean spaces). Based on this theory, we propose \textit{aggregation over compact manifolds} (ACM) to  address the
over-smoothing issue.

Our contributions can be summarized as follows: 

\begin{itemize}[topsep=3pt,parsep=0pt,partopsep=0pt,itemsep=0pt,leftmargin=*]
    \item We propose the notion of \textit{contracted aggregation}  (Section~\ref{sec:method}) and prove that node features will converge to the same single point after many layers of contracted aggregation, leading to over-smoothing.
    We also prove that standard GNNs, e.g., SGC, GCN, GAT, are either \textit{contracted aggregation} or mathematically equal.
    Our claim holds for embedding space of any Riemannian manifolds, including Euclidean Space.
    \item We propose \textit{aggregation over compact manifold} (ACM) (Section~\ref{sec:acm}) that can be integrated into each layer of  GNNs models to prevent contracted aggregation, effectively mitigating over-smoothing across various GNN architectures. 
    \item We provide extensive empirical evidence (Section~\ref{sec:experiments}) that shows  the proposed ACM can alleviate over-smoothing and improve deeper GNNs with better performance compared to the state-of-the-art (SotA) methods. 
    The improvement becomes more significant in a more complex experiment setting (i.e., missing feature) that requires massive layers for GNNs to achieve a good performance.
\end{itemize}

\section{Related Work}
\label{sec:relwork}
\textbf{Over-smoothing in GNNs}   Recent studies have shown that deep stacking of GNN layers can result in a significant performance deterioration, commonly attributed to the problem of over-smoothing~\cite{chen2020measuring,zhao2019pairnorm}.  The over-smoothing issue was initially highlighted in \cite{li2018deeper}, where the authors show that node embeddings will converge to a single point or lose structural information after an infinite number of random walk iterations. Their result has been extended to more general cases considering transformation layers and different activate functions by \cite{DBLP:journals/corr/abs-2008-09864,DBLP:conf/iclr/OonoS20,cai2020note}. 

To address
the over-smoothing issue and enable deeper GNNs, various methods
have been proposed~\cite{hamilton2017inductive,li2018deeper,li2019deepgcns,rong2019dropedge}. 
One approach is to employ skip connections for multi-hop message passing, such as 
GraphSAGE~\cite{hamilton2017inductive} 
JKNet \cite{li2018deeper}.
There are also several approaches developed based on the existing methods in other areas. For instance,
 the study of \cite{li2019deepgcns} leveraged concepts from ResNet~\cite{he2016deep} to incorporate both residual and dense connections in the training of deep GCNs; GCNII~\cite{chen2020simple} utilizes initial residual and identity mapping.
\cite{rong2019dropedge} proposed Dropedge that alleviates over-smoothing through a reduction in message passing by  removing edges inspired by the use of Dropout~\cite{DBLP:journals/jmlr/SrivastavaHKSS14}. 
Another mainstream approach is to normalize output features of each layer, such as BatchNorm,  PairNorm~\cite{zhao2019pairnorm} and DGN  \cite{zhou2020towards}.
DeCorr~\cite{jin2022feature} introduced over-correlation as one reason for significant performance deterioration.  
APPNP~\cite{klicpera2018predict} use Personalized PageRank and GPRGNN~\cite{chien2020adaptive} use Generalized PageRank to mitigate over-smoothing. 
DAGNN~\cite{liu2020towards} was introduced to create deep GNNs by decoupling graph convolutions and feature transformation.
However, these works keep using the original information aggregation function and focus only on alleviating distinguishable after feature transformation.


\noindent\textbf{GNNs on manifolds}  
Recent attention in GNN research has been directed towards neural networks on \emph{manifolds}, impacting various domains like knowledge graph embedding, computer vision, and natural language processing~\cite{DBLP:conf/nips/BalazevicAH19,DBLP:conf/iclr/GulcehreDMRPHBB19,DBLP:conf/cvpr/KhrulkovMUOL20}. Many works focus on the {\emph{hyperbolic space}}, a Riemannian manifold with constant negative curvature, with \cite{chami2019hyperbolic} introducing Hyperbolic Graph Neural Networks (HGNN). \cite{bachmann2020constant} proposes a general GNN on manifolds applicable to hyperbolic space and hyperspheres. This work concentrates on GNN over compact manifolds, encompassing hyperspheres but not hyperbolic spaces.
We introduce a novel GNN framework over compact manifolds. Unlike existing methods, our model doesn't depend on special non-linear functions, such as the \emph{exponential} and \emph{logarithmic} maps~\cite{DBLP:conf/nips/GaneaBH18} defined in hyperbolic space and hyperspheres, making it computationally more straightforward. Additionally, our model is more general, defined on more general kinds of manifold spaces, including hyperspheres.

\section{Preliminaries}
\label{sec:prelim}

\textbf{Compact Metric Space} We start with a brief introduction to basic topology (more details can be found in literature \cite{mendelson1990introduction}). A \emph{metric space} $\mathcal{S}$ is a set equipped with a distance function ${d_\mathcal{S}: \mathcal{S}\times\mathcal{S}\rightarrow \R_{\geq 0}}$ (i.e., a function that is positive, symmetric, and satisfying the triangle inequality). A sequence $p_1,p_2,\ldots\in \mathcal{S}$ \emph{converge} if  
$\lim_{k\rightarrow \infty} \max\{d_\mathcal{S}(p_i, p_j)\mid i, j\geq k\} = 0.$ 
Metric space $\mathcal{S}$ is \emph{closed} if for every convergent sequence $p_1,p_2,\ldots\in \mathcal{S}$ there exits $p\in\mathcal{S}$ such that $\lim_{i\rightarrow\infty} d_\mathcal{S}(p_i, p) = 0$. Metric space $\mathcal{S}$ is  \emph{compact} if 
it is closed and \emph{bounded}---that is, there exists a $r \in \mathbb R$ such that $d_\mathcal{S}(p, q)<r$ for every $p, q\in\mathcal{S}$. 

For example, the open interval $\mathcal{S}_{\sf oint} =(0, 1)\subseteq \R$ is a metric space associated with distance $d_{\mathcal{S}_{\sf oint}}(x, y) = |x-y|$. Then $2^{-1}, 2^{-2}, 2^{-3}, \ldots \in \mathcal{S}_{\sf oint}$ is a convergent sequence, but there is no point in $\mathcal{S}_{\sf oint}$ that is the limit of this sequence. Therefore, $\mathcal{S}_{\sf oint}$ is not closed. In contrast, a unit circle $\mathcal{S}_{\sf circ}$ in $\mathbb R^2$ with $d_{\mathcal{S}_{\sf circ}}(x, y) = ||x-y||$ is closed, since $d_{\mathcal{S}_{\sf circ}}(p, q) < 3$ for all $p, q\in \mathcal{S}_{\sf circ}$, and compact.

\noindent
\textbf{Compact Manifolds} Next we provide a brief introduction to manifolds with the notions necessary for this paper (see further details in literature \cite{Lee_2013}). 
A $d$-dimensional \emph{manifold} $\mathcal{M}$ is a hyper-surface in the Euclidean space  $\mathbb{R}^n$ with $n\geq d$ such that each point has an (open) neighbourhood that is homeomorphic to an open subset of $\mathbb{R}^d$ (i.e., locally looks like $\mathbb{R}^d$).
A \emph{Riemannian manifold} $\mathcal{M}$ is a manifold along with a Riemannian metric, from which one can derive a distance function $d_\M(\textbf{x}, \textbf{y})$ for points $\x, \y \in \mathcal{M}$ thus making $\M$ a metric space (so, we can talk about closed and compact Riemannian manifolds).  
A \textit{compact manifold} is a Riemannian mani\-fold that is also being a compact metric space. Some examples of compact  mani\-folds in $\R^3$ are shown in Fig.~\ref{fig:compact}. 
A differentiable map $f:\mathcal{M}\rightarrow \mathcal{M}$
is \emph{diffeomorphism} if it is bijective and its inverse $f^{-1}$ is a differentiable map.

\noindent
\textbf{Graph Neural Networks} A (undirected) graph  $G=(V, E)$ is a pair of a nodes set $V$ and a edges set $E$. Let $\mathcal{N}(u) = \{v \mid \{u,v\}\in E \}$ and  $\N(u) = \mathcal{N}(u)\cup \{u\}$, ${\bf D} = \mathit{diag}(d_1, \ldots, d_m)$ with $d_{i} = \sum_{j=1}^m {\bf A}_{i,j}$,  where ${\bf A}_{i,j}$ denotes the elements of the adjacency matrix ${\bf A}$, as well as the augmented matrices $\widetilde{\bf A}={\bf A}+ {\bf I}$ and $\widetilde{\bf D}={\bf D} + {\bf I}$, for ${\bf I} = \mathit{diag}(1, \ldots, 1)$. 
We denote the collection of all (undirected) graphs as $\mathcal{G}$.

Given a graph $G$ and an $n$-dimensional node embedding ${\bf H}^{(0)}$ of $G$,  A Graph Neural Network (GNN)  $\mathcal A$ updates the embedding for $\ell$ layers as follows: 
\begin{align}\label{agg_trans}
    \bH^{(k)} &= \mathit{Tran}^{(k)}\left( \mathit{Agg}(\bH^{(k-1)})\right);
\end{align}

\noindent The layer of vanilla-GCN \cite{kipf2016semi} is defined by $\mathit{Agg}(\bH) =\tD^{ -\frac{1}{2}}\tA\tD^{-\frac{1}{2}}\cdot \bH, \mathit{Tran}^{(k)}(\bH) = \sigma(\bH\cdot \W^{(k)})$, where $\sigma$ is activation function and $\W^{(k)}\in \R^{n\times n}$. GNN $\mathcal A$ is over manifold $\mathcal M$ in $\mathbb R^n$ if its aggregation and transformation functions are over $\mathcal M$. 

In this work, we consider arbitrary aggregation functions that are universally defined over the embedding of all neighborhoods $\N(u_i)$ in any graph $G\in\mathcal {G}$. Therefore, all aggregate-transform GNNs with specific aggregations are included, such as GCN~\cite{kipf2016semi} and GAT~\cite{velivckovic2017graph}, as well as SGC~\cite{wu2019simplifying}.

\section{Contracted Aggregation Problem}
\label{sec:method}

 This section is organized as follows.
 First, in Section \ref{sec:contract}, we identify a special property, called \emph{contracted}, of aggregation functions. We show that aggregation functions in GCNs, GATs, and SGCs are contracted or mathematically equivalent. Then, in Section \ref{sec:overSmoothing}, we demonstrate that contracted aggregations cause over-smoothing  by means of Theorem \ref{main_theo}. 
Inspired by this, in Section \ref{sec:construct-non}, we develop a general approach for constructing non-contracted aggregations for alleviating over-smoothing using aggregation over compact Riemannian manifolds (Theorem \ref{theo:agg}). Finally, in Section \ref{sec:aggOnCompact}, we motivate and introduce our aggregation function defined over a specific kinds of compact Riemannian manifolds denoted by $\M_\Wm$.

Next, we let $\M\subset \R^n$ be a $d$-dimensional closed Riemannian manifold, and let $Agg$ be an aggregation function over $\M$. For each graph $G\in\mathcal{G}$ with nodes $\{u_1, \ldots, u_m\}$, 
we denote by
$Agg^G:  \M^m \rightarrow \M^m$
the  \emph{restriction of $Agg$ on $G$} defined by restricting  $Agg$ over embeddings of $G$ on $\M$ (i.e., $\bH\in\M^m$). 

\subsection{Contracted aggregation}\label{sec:contract}
The notion of contracted aggregation, which we introduce in this section, generalizes the usual mean (or avg) aggregation by extracting two center properties of the averaging function over Euclidean space $\R^n$: for all $\x, \x_1, \ldots, \x_h\in \R^n$ we have
\textbf{(i)} $\mathit{Mean}(\x_1, \ldots, \x_h) = \x$ if $\x_1 = \cdots = \x_h = \x$; and
\textbf{(ii)} $d_{\R^n}(\x, \mathit{Mean}(\x_1, \ldots, \x_h)) \leq \max(d_{\R^n}(\x, \x_1), \ldots, d_{\R^n}(\x, \x_h))$.
Fig.~\ref{fig:mean} illustrates the second property. 

\begin{definition}[Contracted aggregation]\label{def:contracted}
For a graph $G$ with nodes $\{u_1, \ldots, u_m\}$, 
$Agg^G$ is \emph{contracted} if for any $1\leq i\leq m$, and for any embedding $\bH\in \M^m$ of $G$ with $\overline{\bH} = \mathit{Agg}^G(\bH)$, the following results hold:
\begin{enumerate}
    \item If $\bH_{j,:} = \x, \forall u_j\in \N(u_i)$, then $\overline{\bH}_{i, :} = \x.$    
    \item For any $ \x\in\M$, we have 
    $d_\M\left(\x,\ \overline{\bH}_{i, :}\right)\leq \max_{u_j\in \N(u_i)}\{d_\M\left(\x, \bH_{j, :}\right)\}$ 
    and the equality holds if and only if Case 1 happens.
\end{enumerate}

\noindent Moreover, we say $\mathit{Agg}^G$  is \emph{equivalently contracted} if there exists a diffeomorphism $g:\M^m\rightarrow \M^m$ s.t.   $g^{-1}\circ\mathit{Agg}^G\circ g$ is contracted, where $\circ$ means the map composition, i.e., $f\circ g(x) = f(g(x))$. 
Finally, an aggregation $\mathit{Agg}$ is \emph{(equivalently) contracted} if  $Agg^G$  is  (equivalently) contracted for any graph $G\in\mathcal{G}$.
\end{definition}

\begin{figure}
  \centering
  \begin{minipage}{0.5\textwidth}
    \centering
   \includegraphics[width = 4.5cm]{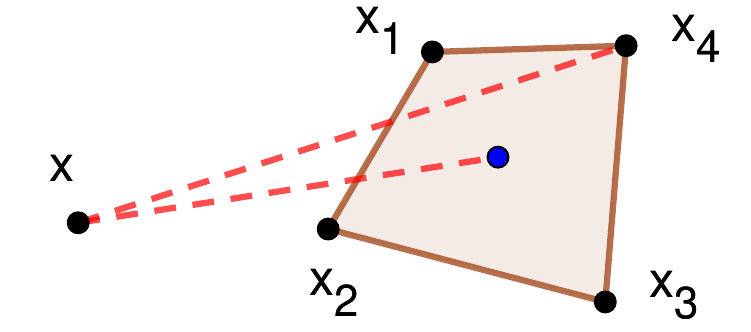}
    \caption{Distance between $\x$ and the mean point (blue) is smaller than $d_{\R^2}(\x, \x_4)$}
    \label{fig:mean}
  \end{minipage}%
  \hfill
  \begin{minipage}{0.45\textwidth}
    \centering
\begin{displaymath}
    \xymatrix@R=5ex@C=13ex{ 
        \M^m \ar[r]^{\mathit{Agg}^G} & \M^m \ar[d]^{g^{-1}} \\
        \M^m \ar[u]^{g}\ar[r]^{g^{-1}\circ\mathit{Agg}^G\circ g} & \M^m 
    }
\end{displaymath}
    \caption{Equivalently contracted aggregation}
    \label{fig:eq_agg}
  \end{minipage}
\end{figure}

\noindent  Intuitively, $\mathit{Agg}^G$ is equivalently contracted if $\mathit{Agg}^G$ is contracted after changing the ``reference frame" by the equivalent transformation $g$ on $\M^m$, which is the space of embedding matrix of $G$, as shown in Fig.\ref{fig:eq_agg}. 
By the following result, we show that the standard aggregations used in SGC, GCN, and GAT are all contracted or equivalently contracted.
\begin{proposition}\label{prop:aggLi}
Let $\M = \R^n$ ($n>0$), assume $\mathit{Agg}$ is a aggregation function on  $\M$ s.t. the restriction of $\mathit{Agg}$ on a graph $G\in\mathcal{G}$ with node $\{u_1, \ldots, u_n\}$ is defined by $\mathit{Agg}^G(\bH) = \bL\cdot\bH,\ \  \forall~ \bH\in \R^{m\times n}$, 
where $\bL\in \R^{m\times m}$. Then we have the following results:
\begin{enumerate}
    \item If $\bL = (1-\lambda)\I+\lambda \tD^{-1}\tA$ for some $\lambda\in(0, 1]$, then $Agg$ is contracted;
    \item If $\bL = (1-\lambda)\I+\lambda \tD^{ -\frac{1}{2}}\tA\tD^{-\frac{1}{2}}$ for some $\lambda\in(0, 1]$, then $Agg$ is equivalently contracted;
    \item If $\bL = \mathit{Att}(\bH)$ is defined by a function $\mathit{Att}:\R^{m\times n}\rightarrow \R^{n\times n}$  s.t. $\bL_{i,j}>0$ and $\sum_{k} \bL_{i,k} = 1$ for any $\bH\in\M^m$ and $1\leq i, j\leq m$, then $Agg$ is contracted.
\end{enumerate}
\end{proposition}
\subsection{Over-smoothing due to contracted aggregations}\label{sec:overSmoothing}
Now, we show that contracted aggregations lead to over-smoothing by Theorem~\ref{main_theo}. This generalizes Theorem 1 of \cite{li2018deeper} to arbitrary equivalently contracted aggregations over Riemannian manifolds,  from aggregations defined by \emph{Laplace smoothing} (aggregations of Case 1 and 2 in Proposition \ref{prop:aggLi}) over Euclidean spaces. 
\begin{theorem}\label{main_theo}
 
Let $\{\bH^{(k)}\}_{k\geq 0}$ be an infinite sequence of embeddings of $G$ over $\mathcal{M}$ s.t. 
$\bH^{(k+1)} = Agg^G\left(\bH^{(k)}\right),\  \text{ for all } k\geq 0.$ 
Assume $G$ is connected. If $\mathit{Agg}^G$ is equivalently contracted wrt a diffeomorphism $g:\M^m\rightarrow \M^m$ (i.e., $g^{-1}\circ\mathit{Agg}^G\circ g$ is contracted), then there exist $\X_0 {=} (\x_0, \ldots, \x_0)\in\M^m$ s.t. $\lim_{k\rightarrow\infty}\bH^{(k)} = g(\X_0)$.
\end{theorem}

\noindent Note that Theorem \ref{main_theo} (proof in Appendix~\ref{app:main_theo})  can be easily extended to non-connected graphs by considering their connected components. Theorem \ref{main_theo} shows that by repeatedly applying an equivalently contracted aggregation, the embedding points of nodes in the given graph  converge to the same point (modulo the impact of the diffeomorphism $g$). That is, equivalently contracted aggregation  leads to over-smoothing.

\subsection{Constructing non-contracted aggregations}\label{sec:construct-non}
Next, we call an aggregation \emph{non-contracted} if it is neither contracted nor equivalently contracted. 
Inspired by Theorem \ref{main_theo} and the example in Figure \ref{fig:intro}, we propose to alleviate the over-smoothing problem by constructing non-contracted aggregations on manifolds.
Indeed, by the following result, we show that any aggregation over a compact manifold is non-contracted. This result provides us with a general approach to constructing non-contracted aggregations. 
\begin{theorem}\label{theo:agg}
Asuume $\mathcal{M}\subset R^n$ is a compact Riemannian manifold  with dimension $d>0$. If a graph $G$ contains a node $u_0$ s.t. $|\N(u_0)|> 2$, 
then for any aggregation $\mathit{Agg}$ over $\mathcal{M}$, the restricted aggregation $\mathit{Agg}^G$ is non-contracted.
\end{theorem}

\noindent For instance, the unit circle is compact, thus, all aggregation over the unit circle is non-contracted. An example of non-contracted aggregation that could avoid over-smoothing problems is shown in Figure \ref{fig:intro}. In contrast, Euclidean spaces $\R^n$ ($n>0$) are not compact and thus can not benefit from the result of Theorem \ref{theo:agg}. The proof of Theorem \ref{theo:agg} is in Appendix~\ref{app:theo_agg}.

\subsection{Our non-contracted aggregation}\label{sec:aggOnCompact}
To the best of our knowledge, there are mainly two kinds of aggregation defined over manifolds beyond Euclidean spaces.  The \emph{tangential aggregations}~\cite{chami2019hyperbolic} and the \emph{$\kappa$-Left-Matrix-Multiplication}~\cite{bachmann2020constant}. However, the former aggregation is defined on \emph{hyperbolic space}, which is not compact. In contrast, the second aggregation could be applied to \emph{hyperspheres}, which are compact Riemannian manifolds. Their definition is based on the \emph{gyromidpoint} introduced in \cite{ungar2010barycentric} and uses \emph{exponential} and \emph{logarithmic}~\cite{DBLP:conf/nips/GaneaBH18} maps defined on hyperspheres. In the following, we propose a simple way to construct aggregations over the specific family of compact manifolds denoted by $\mathcal{M}_{\Wm}$, which include hypersphere as a special case. 
Here, $\Wm\in \R^{n\times n}$ is a \emph{positive-definite matrix}. That is, $\textbf{x}\Wm\textbf{x}^T>0, \forall \x\in \R^n$. 
\begin{definition}[ACM]\label{def:ourAgg}
Let $\Wm\in \R^{n\times n}$ be a positive-definite matrix.
Consider the Riemannian manifold 
$\mathcal{M}_{\Wm} = \{\textbf{x} = (x_1, \ldots, x_{n}) \in \R^{n}\mid \textbf{x}\Wm\textbf{x}^T=1\} \subset \R^n.$ 
ACM aggregation function $\mathit{Agg}$ over $\mathcal{M}_{\Wm}$ is defined as below. For any graph $G\in \mathcal{G}$ with $m$ nodes, the restricted aggregation $\mathit{Agg}^G$ has the form:
\begin{equation}\nonumber
     \mathit{Agg}^G(\bH)=P_\Wm\left(\bL\cdot \bH\right),\ \ \forall ~\bH\in\M^m, 
\end{equation}
\noindent where $ P_\Wm(\x) = \frac{\x}{\sqrt{\textbf{x}\Wm\textbf{x}^T}},\ \ \forall \x\in\R^n\setminus\{\textbf{0}\}$ ($\textbf{0}$ is the original point of $\R^n$),
and $\bL\in \R^{m\times m}$ takes one of the three forms as introduced in Proposition \ref{prop:aggLi}.
\end{definition}
Since $\M_\Wm$ is compact, the above aggregation $\mathit{Agg}$ is non-contracted (Theorem \ref{theo:agg}). For any graph $G\in\mathcal{G}$, $\mathit{Agg}^G$ satisfies condition 1 of Definition \ref{def:contracted} since $P_\Wm(k\cdot \x)  = \x$ for any $k>0, \x\in \M\subset \R_n$. Therefore,  $\mathit{Agg}$ can be regarded as a generalization of mean aggregations on compact manifold $\M_\Wm$. Moreover, $\M_\Wm$ is a hypersphere when $\Wm$ is the identity matrix. Thus, our aggregation also works for hyperspheres. The complexity of computing $ P_\Wm(\x) $  is $O(n^2)$.


\section{Aggregation over Compact Manifolds}
\label{sec:acm}

Here, we introduce ACM (\textbf{A}ggregation over \textbf{C}ompact \textbf{M}anifolds), which integrates our aggregation function over compact manifolds $\M_\Wm$ (Definition \ref{def:ourAgg}) into standard SGC, GCN and GAT. Let $G\in\mathcal{G}$ be a graph with nodes $\{u_1, \ldots, u_m\}$.

\noindent
\textbf{SGC}
We integrate our aggregation into SGC by setting $\bL = \tD^{ -\frac{1}{2}}\tA\tD^{-\frac{1}{2}}$. Then, the embedding is updated by
$
\bH^{(k)}=P_\Wm\left(\tD^{ -\frac{1}{2}}\tA\tD^{-\frac{1}{2}}\cdot \bH^{(k-1)}\right).
$

\noindent
\textbf{GCN}
The adaptation of our aggregation on GCN is defined by Equations (\ref{equ:GCNtrans}) below. 
The aggregation step of GCN is defined the same as above. Next, we show how to build a transformation function between the compact manifold $\mathcal{M}_{\Wm}\subset \R^n$. 

Let
$N_{b} = \{\textbf{x} = (x_1, \ldots, x_n)\in \R^n\mid x_1 = b\}$ 
be a hyperplane in $\R^n$. Then, our transformation function from $\M_\Wm$ to $\M_\Wm$ takes three steps:
\begin{displaymath}
    \xymatrix@C=6ex{ 
        \M_\Wm \ar[r]^{\text{Step 1}} & N_{b}\ar[r]^{\text{Step 2}} &  \R^n  \ar[r]^{\text{Step 3}} & \M_\Wm.
    }
\end{displaymath}

\noindent In Step 1, we map $\mathcal{M}_{\Wm}$ to the hyperplane $N_{b}$ using \emph{push forward} ($\mathit{PF}$) function as illustrated in Fig. \ref{fig:PB}; Then, at Step 2, we map  $N_{b}$ to $\R^n$ by the standard transformation functions (i.e., $\x\mapsto \sigma(\x\cdot\W)$); In the last step, we map $\R^n$ to the manifold $\mathcal{M}_{\Wm}$ by the \emph{push back} ($\mathit{PB}$) function as shown in Fig. \ref{fig:PB}. 
Finally, our GCN layer update $\bH^{(k-1)}$ to a new embedding $\bH^{(k)}$ as below.
\begin{align}
    \overline{\bH}^{(k)}=P_\Wm\left(\tD^{ -\frac{1}{2}}\tA\tD^{-\frac{1}{2}} \cdot  \bH^{(k-1)}\right),\ \  
\bH_u^{(k)}= \mathit{PB}\bigg(\ \sigma\Big( \mathit{PF}(\overline{\bH}^{(k)})\cdot \W^{(k)}\Big)\ \bigg), \label{equ:GCNtrans}
\end{align}

\noindent where $\sigma$ is an activate function, $\W^{(k)}\in \R^{n\times n}$.  The functions $\mathit{PF}$ and $\mathit{PB}$ is a generalization of  \emph{Stereographic projection} and its inverse map from $\M_\Wm$ to $N_{b}$ with center $\x_0 = (a_0, 0, \ldots, 0)\in\M_\Wm$, where $a_0 = \Wm_{11}^{-\frac{1}{2}}$, as in Fig.\ref{fig:PB}.

\begin{figure}[b]
    \centering
    \includegraphics[width = 6cm]{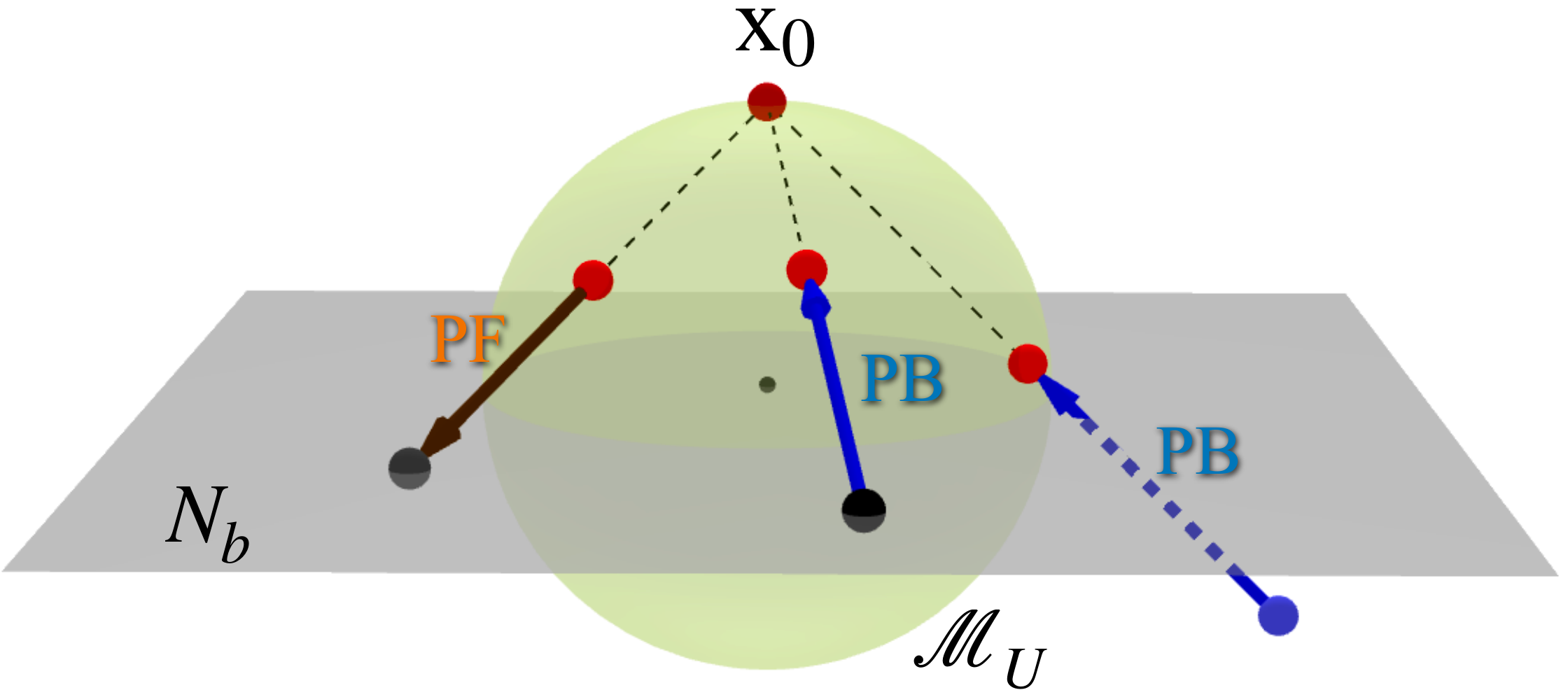}
    \caption{Illustration of PB ({\color{orange}orange}) and PF ({\color{blue}blue}), points in $\M_\Wm$ (resp. $N_b$) are in {\color{red} red} (resp. black). The formal definitions can be found in the supply material.}
    \label{fig:PB}
\end{figure}

\noindent
\textbf{GAT}
Our GAT layer is defined similarly to our GCN layer but with the weighted $\bL = \mathit{Att}(\bH^{(k-1)})$, where $\mathit{Att}$ is the attention function introduced in \cite{velivckovic2017graph}.

\noindent Finally, 
following the works of \cite{he2016deep,li2018deeper}, for node classification task, we use \textit{softmax} classifier in last layer of GNN  to predict the labels of nodes in the given graph $G$.
\section{Experiments} 
\label{sec:experiments}

This section validates the following hypotheses: (1) ACM successfully mitigates over-smoothing in deep layers within GNNs, as evidenced by a reduced decline in performance; (2) ACM effectively alleviate over-smoothing in a more challenging context, specifically the scenario of missing features.

\subsection{Experiments Setup}


\textbf{Datasets} We use four homophilic datasets (Cora, Citeseer, Pubmed~\cite{yang2016revisiting}, and CoauthorCS~\cite{shchur2018pitfalls}) and five heterophilic datasets (Texas, Cornell, Wisconsin, Actor, and Chameleon\cite{pei2020geom}). In the heterophilic datasets, the assumption of neighboring feature similarity is not applicable. Our focus is on standard transductive node classification tasks across these datasets. Details can be found in Appendix~\ref{sec:dataset}.


\noindent
\textbf{Experiment settings}
\label{setting}
We conducted node classification experiments in two settings~\cite{zhao2019pairnorm}: the standard scenario and the \emph{missing feature scenario}. The latter aims to highlight deep Graph Neural Networks' performance by assuming that initial node embeddings in the test and validation sets are absent, replaced with zero vectors. This reflects real-world scenarios, such as in social networks, where new users may have limited connections~\cite{rashid2008learning}, requiring more information aggregation steps for effective representation.


\noindent\textbf{Implementation.} We consider three standard GNN models, namely GCN~\cite{kipf2016semi}, GAT~\cite{velivckovic2017graph}, and SGC~\cite{wu2019simplifying}, and augment them with techniques alleviating oversmoothing: PairNorm~\cite{zhao2019pairnorm}, BatchNorm~\cite{ioffe2015batch}, DropEdge~\cite{rong2019dropedge}, DGN~\cite{zhou2020towards}, and DeCorr~\cite{jin2022feature}. Moreover, we incorporate recent SOTA models, including GPRGNN \cite{chien2020adaptive}, GCNII \cite{chen2020simple}, DAGNN \cite{liu2020towards}, and APPNP~\cite{DBLP:conf/iclr/KlicperaBG19}. To analyze the impact of different methods, we vary the number of layers (5/60/120 for SGC, and 2/15/30/45/60 for GCN and GAT). 
By convention,  SGC typically has more layers than GCN and GAT due to its simpler design and fewer parameters, requiring more layers for best performance and to effectively capture data details. 
We hence evaluate deeper models. Hyperparameter settings for GNN models and optimizers follow previous works~\cite{zhao2019pairnorm,zhou2020towards,jin2022feature}. The hidden layer size is set to 16 units for GCN, GAT, and SGC. GAT has 1 attention head. Training involves a maximum of 1000 epochs with Adam optimizer~\cite{kingma2014adam} and early stopping. GNN weights are initialized using the Glorot algorithm~\cite{glorot2010understanding}. Each experiment is repeated 5 times, and mean values are reported. 
For ACM, we employ $\mathit{tanh}$ as the activation function in hidden layers. We introduce two distinct variants of the ACM, differentiated by the configuration of the parameter matrix $\Wm$. (i)  the first variant, still denoted as $\text{ACM}$, sets the matrix $\Wm$ to the identity matrix $\mathbf{I}$; (ii) the second variant, denoted as ACM*, defines the matrix $\Wm$ as a positive definite diagonal matrix, where each diagonal element is a trainable parameter. Therefore, for $\text{ACM}$,  the underlying manifold is the unit hypersphere, while ACM* offers a variable manifold (e.g., ellipsoid) that can be tailored through training to better suit the underlying data distribution.


\subsection{Experiment Results}
In this section, we compare the performance of ACM with previous methods that tackle the over-smoothing issue on SGC, GCN, and GAT. Refer to \cite{zhou2020towards,jin2022feature}  for all the results of the methods other than ACM 
unless stated otherwise.




\noindent
\textbf{Mitigate the performance drop in deep SGC.}
\label{sec:sgc_result} 
SGC is a proper model to show the effectiveness of ACM in relieving over-smoothing. As SGC has only one feature transformation step, the performance of SGC highly depends on the information aggregation steps.

We present the performances of SGCs with 5/60/120 layers.  \noindent Table~\ref{table:sgcwf}  summarizes the results of applying   different over-smoothing alleviation approaches to SGC. “None” indicates the vanilla SGC without over-smoothing alleviation.  
As shown in Table~\ref{table:sgcwf}, ACM  outperforms the other over-smoothing alleviation methods in most of the cases. In particular, on the Citeseer dataset with 120 layers, the improvements over None, BatchNorm, PairNorm, and DGN achieved by ACM are by a margin of 57.5\%, 19\%, 19.8\%,  and 14.3\%, respectively.

\noindent  To show the detailed performance differences on each layer, we illustrate in Figure \ref{fig:test_acc_SGC_MS1_all}(a,b) the test accuracy of SGC trained with different methods with layers from 1 to 120 on datasets Cora and CoauthorCS. 
We can see that ACM behaved the best in slowing down the performance drop on Cora  and comparably with the best system DGN on CoauthorCS. The same conclusion holds for other datasets  as on Cora (see Appendix \ref{appendix:figures}).

\begin{table}
  \renewcommand*{\arraystretch}{1.1}
  \setlength{\tabcolsep}{1.5mm}
  \centering
  \caption{\small{Node classification accuracy (\%) on SGC. The two best performing methods are highlighted in \textcolor{blue}{\textbf{blue}} (First), \textcolor{violet}{\textbf{violet}}  (Second)}.}
  \label{table:sgcwf}
  \resizebox{1.0\columnwidth}{!}{
    \begin{tabular}{c|ccc|ccc|ccc|ccc}
      \hline
      \textbf{Datasets} & \multicolumn{3}{c}{Cora} & \multicolumn{3}{|c}{Citeseer} & \multicolumn{3}{|c}{Pubmed} & \multicolumn{3}{|c}{CoauthorCS} \\
      \hline
      \textbf{\#Layers} & 5 & 60 & 120 & 5 & 60 & 120 & 5 & 60 & 120 & 5 & 60 & 120 \\
      \hline

    None & 75.8 & 29.4 & 25.1 & \textcolor{blue}{\textbf{69.6}} & \textcolor{violet}{\textbf{66.3}} & 9.4 & 71.5 & 34.2 & 18.0 & 89.8 & 10.2 & 5.8 \\
    BatchNorm & 76.3 & 72.1 & 51.2 & 58.8 & 46.9 & 47.9 & 76.5 & 75.2 & 70.4 & 88.7 & 59.7 & 30.5 \\
    PairNorm & 75.4 & 71.7 & 65.5 & 64.8 & 46.7 & 47.1 & 75.8 & 77.1 & 70.4 & 86.0 & 76.4 & 52.6 \\
    DGN & \textcolor{violet}{\textbf{77.9}} & \textcolor{violet}{\textbf{77.8}} & \textcolor{violet}{\textbf{73.7}} & \textcolor{violet}{\textbf{69.5}} & 53.1 & \textcolor{violet}{\textbf{52.6}} & \textcolor{violet}{\textbf{76.8}} & \textcolor{violet}{\textbf{77.4}} & \textcolor{violet}{\textbf{76.9}} & \textcolor{violet}{\textbf{90.2}} & \textcolor{violet}{\textbf{81.3}} & \textcolor{violet}{\textbf{60.8}} \\
    ACM & \textcolor{blue}{\textbf{78.5}} & \textcolor{blue}{\textbf{80.0}} & \textcolor{blue}{\textbf{77.0}} & 65.0 & \textcolor{blue}{\textbf{67.4}} & \textcolor{blue}{\textbf{66.9}} & \textcolor{blue}{\textbf{76.9}} & \textcolor{blue}{\textbf{78.3}} & \textcolor{blue}{\textbf{78.6}} & \textcolor{blue}{\textbf{90.9}} & \textcolor{blue}{\textbf{80.1}} & \textcolor{blue}{\textbf{57.4}} \\
    \hline
    \end{tabular}}
\end{table}

\noindent
\textbf{Mitigation of the performance drop in deeper GCN, GAT.}
\label{sec:gcngat_res} We integrated our ACM method into  GCN and GAT and measured 
the performance of different GNNs with 30/45/60 layers in Table~\ref{table:gatgcn}. The full table is in ~\ref{sec:Tablesappendix}. The results of None, BatchNorm, PairNorm, DGN at 45/60 layers were obtained by running the source code supplied by~\cite{zhou2020towards}. The results of Dropedge and Decorr at 45/60 layers were collected by running the source code from~\cite{jin2022feature}.

From Table~\ref{table:gatgcn}, we can see that ACM can greatly improve the performance of deeper GNNs on these datasets. 
In particular, given the same number of layers, ACM consistently achieves the best performance for almost all cases, dramatically outperforms recent baselines on deep models (30/45/60 layers), and keeps the comparable performance in the lower-layers model. 
For instance, on the Cora dataset, at 60 layers of GCN setting, ACM improves the classification accuracy of None, BatchNorm, PairNorm, Dropedge, DGN and DeCorr by 56.5\% 11.1\%, 22.5\%, 56.5\%, 16.9\%, 56.5\%, respectively.

\begin{table}
\renewcommand*{\arraystretch}{1.1}
\footnotesize
\setlength{\tabcolsep}{1.5mm}
\centering
\caption{\small{Node classification accuracy (\%) on GCN, GAT.}}
\label{table:gatgcn}
\resizebox{\textwidth}{!}{
\begin{tabular}{c|c|ccc|ccc|ccc|ccc}
\hline

\multicolumn{2}{c|}{Datasets} &      \multicolumn{3}{c}{Cora} &     \multicolumn{3}{|c}{Citeseer} &    \multicolumn{3}{|c}{Pubmed} &   \multicolumn{3}{|c}{CoauthorCS}    \\ \hline
\textbf{Methods} & \textbf{\#Layers  } & L30 & L45  &    L60 & L30 & L45  &    L60 & L30 & L45  &    L60 & L30 & L45  &    L60  \\ \hline
 \multirow{9}*{\textbf{GCN}}  &    None & 13.1  & 13 & 13.0  & 9.4  & 7.7 & 7.7  & 18.0  & 18 & 18.0  & 3.3  & 3.3 & 3.3   \\
 &   BN & 67.2  & 60.2 & 58.4  & 47.9  & 38.7 & 36.5  & 70.4  & 72.9 & 67.1  & 84.7  & 80.1 & 79.1   \\
 &   PN & 64.3  & 54.5 & 47.0  & 47.1  & 43.1 & 37.1  & 70.4  & 63.4 & 60.5  & 64.5  & 70 & 66.5   \\
 &   DropEdge    & 45.4  & 13 & 13.0  & 31.6  & 7.7 & 7.7  & 62.1  & 18 & 18.0  & 31.9  & 3.3 & 3.3   \\
 &   DGN & 73.2  & 67.8 & 52.6  & 52.6  & 45.8 & 40.5  & \textcolor{violet}{\textbf{76.9}} & 73.4 & 72.8  & 84.4  & 83.7 & \textcolor{violet}{\textbf{82.1} }  \\
 &   DeCorr & \textcolor{violet}{\textbf{73.4}} & 38.9 & 13.0  & \textcolor{blue}{\textbf{67.3}} & 37.1 & 7.7  & \textcolor{blue}{\textbf{77.3}} & 32.5 & 13.0  & \textcolor{violet}{\textbf{84.5}} & 29 & 3.3   \\
 &   ACM & \textcolor{blue}{\textbf{73.5}} & \textcolor{blue}{\textbf{71.6}} & \textcolor{violet}{\textbf{69.5}} & 55.1  & \textcolor{violet}{\textbf{56.5}} & \textcolor{blue}{\textbf{53.5}} & 75.7  & \textcolor{violet}{\textbf{74.3}} & \textcolor{violet}{\textbf{74.4}} & \textcolor{blue}{\textbf{85.5}} & \textcolor{violet}{\textbf{84.6}} & \textcolor{blue}{\textbf{82.4}}  \\
 &   ACM* & 72.4  & \textcolor{blue}{\textbf{71.6}} & \textcolor{blue}{\textbf{70.3}} & \textcolor{violet}{\textbf{56.5}} & \textcolor{blue}{\textbf{57.0}} & \textcolor{violet}{\textbf{53.4}} & 74.4  & \textcolor{blue}{\textbf{76.3}} & \textcolor{blue}{\textbf{74.8}} & 84.3  & \textcolor{blue}{\textbf{84.7}} & 70.7   \\ \hline
 \multirow{9}*{\textbf{GAT}}  &    None & 13.0  & 13 & 13.0  & 7.7  & 7.7 & 7.7  & 18.0  & 18 & 18.0  & 3.3  & 3.3 & 3.3   \\
 &   BN & 25.0  & 21.6 & 16.2  & 21.4  & 21.1 & 18.1  & 46.6  & 45.3 & 29.4  & 16.7  & 4.2 & 2.6   \\
 &   PN & 30.2  & 28.8 & 19.3  & 33.3  & 30.6 & 27.3  & 58.2  & 58.8 & 58.1  & 48.1  & 30.4 & 26.6   \\
 &   DropEdge    & 51.0  & 13 & 13.0  & 36.1  & 7.7 & 7.7  & 64.7  & 18 & 18.0  & 52.1  & 3.3 & 3.3   \\
 &   DGN & 51.3  & 44.2 & 38.0  & 45.6  & 32.8 & 27.5  & 73.3  & 53.7 & 60.1  & 75.5  & 20.9 & 44.8   \\
 & DeCorr & 54.3  & 18.3 & 13.0  & 46.9  & 18.9 & 7.7  & 74.1  & 48.9 & 18.0  & 77.3  & 19.2 & 3.3   \\
 & ACM & \textcolor{violet}{\textbf{67.4}} & \textcolor{violet}{\textbf{53.5}} & \textcolor{violet}{\textbf{48.5}} & \textcolor{violet}{\textbf{49.5}} & \textcolor{violet}{\textbf{38.8}} & \textcolor{violet}{\textbf{38.4}} & \textcolor{violet}{\textbf{75.4}} & \textcolor{violet}{\textbf{72.0}} & \textcolor{violet}{\textbf{68.4}} & \textcolor{violet}{\textbf{84.8}} & \textcolor{violet}{\textbf{79.8}} & \textcolor{violet}{\textbf{74.2}}  \\
 & ACM* & \textcolor{blue}{\textbf{71.4}} & \textcolor{violet}{\textbf{53.6}} & \textcolor{blue}{\textbf{61.3}} & \textcolor{blue}{\textbf{56.2}} & \textcolor{blue}{\textbf{48.3}} & \textcolor{blue}{\textbf{47.2}} & \textcolor{blue}{\textbf{76.7}} & \textcolor{blue}{\textbf{75.0}} & \textcolor{violet}{\textbf{68.5}} & \textcolor{blue}{\textbf{85.0}} & \textcolor{blue}{\textbf{82.6}} & \textcolor{blue}{\textbf{76.3}}  \\ \hline

\end{tabular}
}
\end{table}

\begin{figure}
\centering
\includegraphics[width=\linewidth]{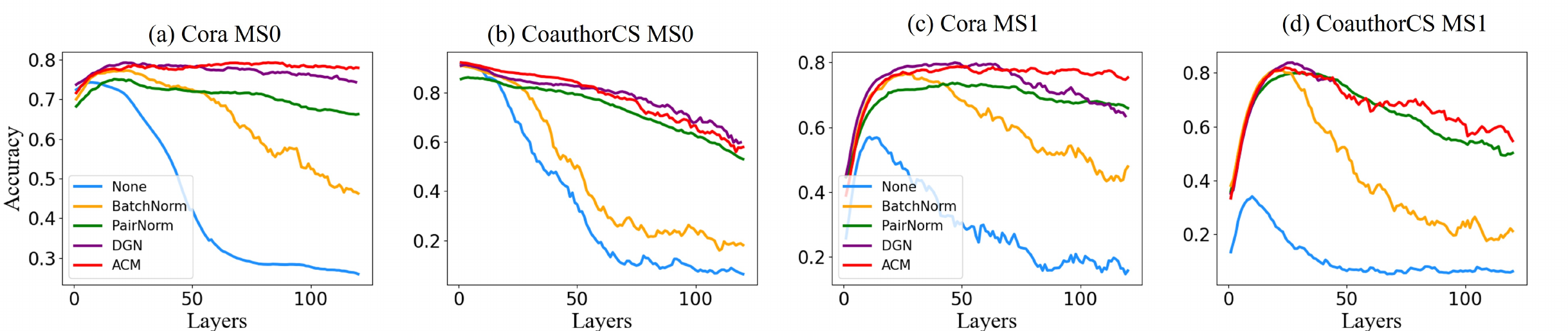}
\caption{Test accuracy of SGC with different methods and layers on Cora and CoauthorCS datasets under miss rate 0 (MS0) and miss rate 1 (MS1) scenario.}
    \label{fig:test_acc_SGC_MS1_all}
\end{figure}

\begin{table}
\renewcommand*{\arraystretch}{1.1}
\footnotesize
	\center
    \caption{
    Test accuracy (\%) on missing feature setting.} 
    \label{table:missrate}	
	\resizebox{1\columnwidth}{!}{
   \setlength{\tabcolsep}{1pt}
          \begin{tabular}{c|cccc|cccc|cccc}
    \hline
    \textbf{Model} & \multicolumn{4}{c}{SGC} & \multicolumn{4}{|c}{GCN} & \multicolumn{4}{|c}{GAT} \\ \hline

        \textbf{Methods} & Cora & Citeseer & Pubmed  & CCS & Cora & Citeseer & Pubmed  & CCS & Cora & Citeseer & Pubmed  & CCS  \\ \hline
    None & 63.4(5) & 51.2(40) & 63.7(5) & 71(5) & 57.3(3) & 44(6) & 36.4(4) & 67.3(3) & 50.1(2) & 40.8(4) & 38.5(4) & 63.7(3) \\
    BN & 78.5(20) & 50.4(20) & 72.3(50) & 84.4(20) & 71.8(20) & 45.1(25) & 70.4(30) & 82.7(30) & 72.7(5) & 48.7(5) & 60.7(4) & 80.5(6) \\
    PN & 73.4(50) & 58(120) & 75.2(30) & 80.1(10) & 65.6(20) & 43.6(25) & 63.1(30) & 63.5(4) & 68.8(8) & 50.3(6) & 63.2(20) & 66.6(3) \\
    DGN & \textcolor{violet}{\textbf{80.2(5)}} & \textcolor{violet}{\textbf{58.2(90)}} & \textcolor{violet}{\textbf{76.2(90)}} & \textcolor{violet}{\textbf{85.8(20)}} & 67(6) & 44.2(8) & 69.3(6) & 68.6(4) & 67.2(6) & 48.2(6) & 67.2(6) & 75.1(4) \\
    ACM & \textcolor{blue}{\textbf{80.7(90)}} & \textcolor{blue}{\textbf{59.5(109)}} & \textcolor{blue}{\textbf{76.3(91)}} & \textcolor{blue}{\textbf{86.5(26)}} & \textcolor{violet}{\textbf{76.3(20)}} & \textcolor{blue}{\textbf{50.2(30)}} & \textcolor{violet}{\textbf{72(30)}} & \textcolor{violet}{\textbf{83.7(25)}} & \textcolor{violet}{\textbf{75.8(8)}} & \textcolor{blue}{\textbf{54.5(5)}} & \textcolor{violet}{\textbf{72.3(20)}} & \textcolor{violet}{\textbf{83.6(15)}} \\
    ACM* & \textbf{-} & \textbf{-} & \textbf{-} & \textbf{-} & \textcolor{blue}{\textbf{73.8(20)}} & \textcolor{violet}{\textbf{49.1(30)}} & \textcolor{blue}{\textbf{73.3(15)}} & \textcolor{blue}{\textbf{84.3(20)}} & \textcolor{blue}{\textbf{72.8(15)}} & \textcolor{violet}{\textbf{46.5(6)}} & \textcolor{blue}{\textbf{72.4(15)}} & \textcolor{blue}{\textbf{83.7(15)}} \\ \hline

  \end{tabular}
               }
\end{table}

\begin{table}[t]
	\center
    \caption{Node classification accuracy (\%) for eight datasets} 
    \label{table:eight_dataset}	 
    	\resizebox{1.0\columnwidth}{!}{
        \begin{tabular}{c|ccc|c c c |cc }
        \hline      
      \textbf{Methods} & \textbf{Cora} &  \textbf{Citeseer} &   \textbf{Pubmed} & \textbf{Texas} &  \textbf{Cornell} &   \textbf{Wisconsin}&   \textbf{Actor} &\textbf{Chameleon} \\
      \hline
      GCN& 82.3&71.7&78.6&67.2&57.4&56.0 &26.7 & 43.3\\ 
      GAT & 81.2 & 69.5 & 77.4 &71.6&63.6&64.1&27.8  & 48.1 \\ 
      DAGNN & 84.4 & \color{blue}\textbf{73.3} &  \color{violet}\textbf{80.3} & 78.6 & 68.1  & 71.3 & 32.4 & 47.9  \\
    APPNP & 83.3 &  71.8 &  80.1 &  73.5 & 54.3  &  65.4 & 34.5 &  \color{blue}\textbf{54.3} \\
      GPRGNN & 83.0 & 71.3 & 71.5 & \color{violet}\textbf{78.7} & 76.7 & 71.2 & 37.1  & 48.8  \\
      GCNII & \color{violet}\textbf{84.5}  & 71.2 & 79.1 & 67.2 & \color{violet}\textbf{80.8} & \color{violet}\textbf{72.1} &  \color{violet}\textbf{33.4} & \color{violet}\textbf{51.2} \\
      \hline
      GPR+ACM & \color{blue}\textbf{84.7} & \color{violet}\textbf{71.9} & \color{blue}\textbf{80.4} &   \color{blue}\textbf{80.3} & \color{blue}\textbf{87.3}  & \color{blue}\textbf{88.7} & \color{blue}\textbf{38.8}  &  50.8\\
      \hline

 \end{tabular}
 }
\end{table}
\noindent
\textbf{Enabling deeper GNNs under missing feature setting.} Table~\ref{table:sgcwf} and ~\ref{table:gatgcn} show that deeper GNNs often perform worse than shallower ones. The ``missing feature setting” experiments (see Section~\ref{setting}) executed the node classification task in  more complex scenario. Such  complex setting requires massive layers for GNNs to achieve good performance, for which ACM shows particular benefits since over-smoothing issue becomes severer with increasing layers.


In Table~\ref{table:missrate}, 
Acc is the best test accuracy of a model with the optimal layer number \#K. 
We observe that ACM outperforms the other over-smoothing alleviation methods in almost all cases.
The average accuracy improvements achieved by ACM over None, BatchNorm, PairNorm, and DGN are 19.62\%, 5.4\%, 9.3\%, and 1.18\%, respectively. 
Moreover, the best performances  were achieved with large layers. 
For instance, on the Cora dataset, ACM behaved  the best with 90, 34, and 11 layers for SGC, GCN, and GAT, respectively, 
evidently higher than those widely-used  shallow vanilla models, i.e., with 2 or 3 layers. 

To better compare 
different methods  with various layers,  
we run the experiments for all the layers from 1 to 120 on SGC. The result is illustrated by Figure \ref{fig:test_acc_SGC_MS1_all}(c,d).
We can see that ACM achieved best performance on deep SGCs on Cora and Citeseer. The same conclusion holds for other datasets (see Appendix \ref{appendix:figures}).

\noindent\textbf{Combining with Other Deep GNN Methods} 
To ascertain whether ACM can function as a supplementary technique for alleviating over-smoothing, we choose to integrate it with one of the strongest baselines, GPRGNN.
Our experiments on eight benchmark datasets, with the average accuracy reported for 10 random seeds in Table~\ref{table:eight_dataset}, shows that ACM consistently enhances GRPGNN across all datasets. For example, ACM enhance the GPRGNN by 10.6\% on Cornell and 17.5\% on Wisconsin, respectively. This supports the idea that ACM  can enhance model performance.

\section{Conclusion}
\label{sec:conclusion}

This paper proposed a general theory for explaining over-smoothing in Riemannian manifolds, which attributes the over-smoothing to the contracted aggregation problem.
Inspired by this theoretical result, we proposed a general framework ACM---by constructing non-contracted aggregation functions on compact 
Riemann manifold---to
effectively mitigate the over-smoothing and encourage deeper GNNs to learn distinguishable node embeddings. The effectiveness of ACM 
has been demonstrated through the extensive experimental results.

In the future, we aim to extend our theoretical findings to encompass   transformation functions. Additionally, we intend to employ the ACM framework in tasks requring deeper GNNs, such as multi-hop question answering.

\bibliographystyle{splncs04}

\newpage
\appendix
\onecolumn

\section{Details of our model}
\paragraph{$\mathit{PF}$ and $\mathit{PB}$ }
 The formal definitions of $\mathit{PF}$ and $\mathit{PB}$ are given below:
\begin{align*}
    \mathit{PF}(\textbf{w}) &= \left(\frac{b-a_0}{w_1-a_0}\right)\cdot (\textbf{w}-\textbf{x}_0)+\textbf{x}_0,\ \ \textbf{w}{=}(w_1, {\ldots}, w_n),\\
\mathit{PB}(\textbf{v}) &= \left(\frac{-2 (\textbf{v}-\textbf{x}_0)\Wm\textbf{x}_0^T}{(\textbf{v}-\textbf{x}_0)\Wm(\textbf{v}-\textbf{x}_0)^T}\right)\cdot(\textbf{v}-\textbf{x}_0)+\textbf{x}_0.
\end{align*}

\paragraph{GAT}Our GAT layer is defined similarly to our GCN layer but with the weighted $\bL = \mathit{Att}(\bH^{(k-1)})$ as follows:
\begin{equation}\nonumber
\overline{\bH}_u^{(k)}=P_\Wm\Big(\mathit{Att}(\bH^{(k-1)})\cdot \bH^{(k-1)}\Big),\ \ 
\bH_u^{(k)}=\mathit{PB}\bigg(\ \sigma\Big( \mathit{PF}(\overline{\bH}_u^{(k)})\cdot \W^{(k)}\Big)\ \bigg).
\end{equation}

\paragraph{Classification}

 Let $\hat{Y} \in R^{m \times C}$ be the predictions of $C$ classes for $m$ nodes, where $\hat{Y}_{ic}$ denotes the probability of node $u_i$ belongs to class $c$.
Then, the class prediction $\hat{Y}$ of a $\ell$-layer GNN is obtained by $ \hat{Y} =\mathit{softmax}(\bH^{(\ell)})$, where $ \mathit{softmax}({\bH^{(\ell)}_i}) =\frac{\bH^{(\ell)}_i}{\sum\nolimits_{c=1}^{C}\mathit{exp}(\bH^{(\ell)}_{ic})}$.

\section{Proofs}
\subsection{Proof of Theorem \ref{main_theo}}
\label{app:main_theo}


For any $\x\in\M$, define
$$f_{\x} (\bH) = \max_{1\leq i \leq \numNode}(d(\x, \bH_{i,:})).$$
Then $f_{\x}$ is a continuous function. Since $Agg^G$ is contracted, we have 
$$d(\x, \bH_{i,:}^{(k)})\leq \max_{1\leq j \leq \numNode}d(\x, \bH_{j,:}^{(k-1)}) = f_{\x}(\bH^{(k-1)}),$$
and thus
$$0\leq f_{\x} (\bH^{(k)}) = \max_{1\leq i \leq \numNode}(d(\x, \bH_{i,:}))\leq f_{\x}(\bH^{(k-1)}).$$
Therefore, the sequence of numbers $\{f_{\x}(\bH^{(k)})\}_{0\leq k}$ is a bounded and monotonically decreasing sequence.
Let $$F(\x) = \lim_{k\rightarrow\infty}f_{\x}(\bH^{(k)}).$$

First, we introduce some lemmas.
\begin{lemma}\label{lemma:equalNeigbour}
Assume that $F(\x)>0$ for a point $\x\in\M$, and assume
$0\leq k_1< k_2<\ldots<k_t<\ldots$
is an infinite sequence such that for any $1\leq i \leq \numNode$, $$\lim_{t\rightarrow\infty}\bH_{i,:}^{(k_t)} = \x_{i}$$ for some $\x_{i}\in\M$. Then for any  $1\leq i \leq \numNode$, there exists $\y_0\in\M$ such that $\x_{j} = \y_0, \forall u_j\in\N(u_i)$.
\end{lemma}

\begin{proof}
For all $1\leq i \leq \numNode$, the map from $\{\x_{j}\}_{u_j\in\N(u_i)}$ to $\x_{i}'$ defined below is a differentiable map.
\begin{equation}\label{equa:agg_{i,:}0}
\x_{i}' = Agg^G(\x_{j}: \forall u_j\in\N(u_i))
\end{equation}
 Since $\mathcal{M}$ and $\M^N$ are closed, the map defined by Equation (\ref{equa:agg_{i,:}0}) is locally \emph{Lipschitz map} around any point $\bH\in \M^N$.

Then, by assumption of the Lemma, for any $1\leq i \leq \numNode$, there exists $\x_{i} =\x_{i}^0, \x_{i}^1, \x_{i}^2, \ldots, \x_{i}^s,\ldots$ such that
$$\lim_{t\rightarrow\infty}\bH_{i,:}^{(k_t+s)} = \x_{i}^s,\ \ \forall s>0$$

where
\begin{equation}\label{equa:agg_{i,:}1}
\x_{i}^s = Agg^G(\x_{j}^{s-1}: \forall u_j\in\N(u_i))
\end{equation}

Next, we prove the lemma by contradiction.

Assume that $u_0$ is a node such that $\{\x_{j}\}_{u_j\in\N(u_0)}$ are not identical. Then, since $Agg^G$ is contracted, for any $\x\in\M$, we have 
$$d(\x, \x_{0}^1)< \max_{u_j\in\N(u_0)}\{d(\x, \x_{j}^0)\}\leq F(\x).$$
For any $1\leq n \leq \numNode$, since $G$ is connected, we can find a path from $u_0$ to $u_n$ in $G$. Without loss of generality, we assume that the path is the following sequence of nodes:
$$u_0, u_1, \ldots, u_n \ \ (n\leq \numNode),$$
where $u_j\in\N(u_{j+1})$. 
Since $d(\x, \x_{0}^1)< F(\x)$, $u_0\in \N(u_{1})$ and $\max_{u_j\in\N(u_1)}\{d(\x, \x_{j}^0)\}\leq F(\x)$, we have
$$d(\x, \x_{u_1}^2)< \max_{u_j\in\N(u_1)}\{d(\x, \x_{j}^1)\}\leq F(\x).$$
Repeat this progress, we conclude that
$$d(\x, \x_{u_s}^{s+1})< \max_{u_j\in\N(u_{s})}\{d(\x, \x_{j}^{s})\}\leq F(\x),\ \ \forall s>0.$$

Therefore, we have $d(\x, \x_{i}^s)<F(\x)$ when $s>0$.
Since $\lim_{t\rightarrow\infty}\bH_{i,:}^{(k_t+\numNode+1)} = \x_{i}^{\numNode+1}$,there exists $T$ such that:
$$d(\x, \bH_{j,:}^{(k_t+\numNode+1)})< F(\x),\ \ \forall t>T, 1\leq i \leq \numNode.$$
Then we have 
$$F(\x)\leq f_\x(\bH^{(k_t+\numNode+1)})= \max_{1\leq j \leq \numNode}d(\x, \bH_{j,:}^{(k_t+\numNode+1)}) <F(\x),\ \ \forall t>T.$$
This is a contradiction!

The lemma is proved by contradiction.

\end{proof}

\begin{lemma}\label{lemma:graph}
Given a connected graph $G$ and an embedding $\bH$ of $G$ on $\M$. If for any node $1\leq i \leq \numNode$ in $G$, the  embedding points of neighbours of $u$ are identical (i.e.,  $\{\x_{j}\}_{u_j\in\N(u_i)}$ are identical), then there exists $\x \in \M$ s.t. $\bH_{i,:} =\x, \forall 1\leq i\leq m$. 

\end{lemma}
\begin{proof}
If $\N = \N$, then any node $n$ is in its neighborhood, and thus we have $\bH$ consisting of one point since $G$ is connected.

\end{proof}

Now, we start proving Theorem \ref{main_theo}.

\begin{proof}(proof of Theorem \ref{main_theo})
First, we assume that $Agg^G$ is contracted.

Since $\M$ is closed, $\M^N$ also closed. Recall that $f_{\x} (\bH^{(k)}) = \max_{1\leq i \leq \numNode}(d(\x, \bH_{i,:}^{(k)}))$ is a monotonically decreasing sequence and thus $\bH^{(k)}\in \M^N$ is uniformly bounded. Then, there exists an infinite sequence 
$$0\leq k_1< k_2<\ldots<k_t<\ldots$$ 
such that for each $1\leq i \leq \numNode$, there exists $\x_{i}\in\M$ such that $\lim_{t\rightarrow\infty}\bH_{i,:}^{(k_t)} = \x_{i}$. 

Let $$F(\x) = \lim_{k\rightarrow\infty}f_{\x}(\bH^{(k)}).$$ If $F(\x)=0$ for some $\x\in\M$, then we have $\lim_{k\rightarrow\infty}\bH_{i,:}^{(k)} = \x_{i}$  for any $1\leq i \leq \numNode$. This proves the theorem.

Next, we assume $F(\x)>0$ for all $\x\in\M$. Then, by Lemmas \ref{lemma:equalNeigbour} and \ref{lemma:graph}, we know the set $\{\x_{i}\}_{1\leq i \leq \numNode}$ consist of one point $\x_0$.  Next, we prove $\lim_{k\rightarrow\infty}\bH_{i,:}^{(k)} = \x_0$ by contradiction.

Otherwise, there exists node, assume it is $u_0$, such that  $\lim_{k\rightarrow\infty}\bH_{0,:}^{(k)} \neq \x_0$. 
Then, there exists $\epsilon>0$, and we can find an infinite sequence 
$$0\leq k'_1< k'_2<\ldots<k'_t<\ldots$$ 
such that $$d(\bH_{0,:}^{(k'_t)}, \x_0)>\epsilon,  \ \ \forall t\geq 1.$$ 

Moreover, there exists a sub-sequence $\{k_s''\}_{0\leq s}$ of $\{k'_t\}_{0\leq t}$ such that $\bH_{j,:}^{(k_s'')}$ converged for any $1\leq j \leq \numNode$. For simplicity, we assume $\{k_s''\}_{0\leq s} = \{k'_t\}_{0\leq t}$. Then, for any $1\leq i \leq \numNode$, we have $\lim_{t\rightarrow\infty}\bH_{i,:}^{(k'_t)} = \x_{i}'$ for some $\x_{i}'\in\M$.
By Lemmas \ref{lemma:equalNeigbour} and \ref{lemma:graph}, we know the set $\{\x_{i}'\}_{1\leq i \leq \numNode}$ consist of one point $\x'_0$. 

Since 
$$F(\x) = \lim_{k\rightarrow\infty}f_{\x}(\bH^k) = \lim_{t\rightarrow\infty}f_{\x}(\bH^{k_t}) = \lim_{t\rightarrow\infty}f_{\x}(\bH^{k'_t}),$$
we have $d(\x, \x'_0)=d(\x, \x_0), \ \ \forall \x\in\M.$
Then we must have $\x_0 = \x_0'$. This is contradict to the fact that $\lim_{k\rightarrow\infty}\bH_{0,:}^{(k)} \neq \x_0$.

Finally, we have $\lim_{k\rightarrow\infty}\bH_{i,:}^{(k)} = \x_0$ when $\mathit{Agg}^G$ is contracted.

Now, if  $\mathit{Agg}^G$ is equivalently contracted with a  diffeomorphism $g$. That is, $g^{-1}\circ Agg^G\circ g$. Then, by the above result, there exists $\x_0\in\M$ s.t. for $\X_0 = (\x_0, \ldots, \x_0)\in\M^\numNode$, we have
\begin{align*}
 \X_0 &= \lim_{k\rightarrow\infty}\overbrace{(g^{-1}\circ Agg^G\circ g)\circ \ldots\circ (g^{-1}\circ Agg^G\circ g)}^{k\text{ times}}(g^{-1}(\bH^{(0)}))\\
    &= \lim_{k\rightarrow\infty}g^{-1}\circ \overbrace{Agg^G\circ \ldots\circ Agg^G}^{k\text{ times}}(\bH^{(0)}).
\end{align*}
Therefore,  we have 
$\lim_{k\rightarrow\infty}\bH^{(k)} = \lim_{k\rightarrow\infty}\overbrace{Agg^G\circ \ldots\circ Agg^G}^{k\text{ times}}(\bH^{(0)}) = g(\X_0).$
\end{proof}








\subsection{Proof of Theorem \ref{theo:agg}}
\label{app:theo_agg}
\begin{proof}
For any graph $G\in\mathcal{G}$ with nodes $\{u_1, \ldots, u_m\}$, if the restricted $\mathit{Agg}^G$ of $G$ on $\mathcal{M}$ do not satisfies condition 1 of Definition \ref{def:contracted}, then there is nothing to prove. Next, we assume $\mathit{Agg}^G$  satisfies condition 1 of Definition \ref{def:contracted}. If we can show that $\mathit{Agg}^G$ does not satisfies condition 2 Definition \ref{def:contracted}. Then $\mathit{Agg}^G$ is not contracted. Then, $\mathit{Agg}^G$ is not equivalently contracted since we can apply the same arguments above on $g^{-1}\circ \mathit{Agg}^G\circ g$ for any diffeomorphism $g:\M^m\rightarrow \M^m$. Therefore, $Agg^G$ is not equivalently contracted and thus non-contracted.

Next,  we show that if  $\mathit{Agg}^G$  satisfies condition 1 of Definition \ref{def:contracted}, then  $\mathit{Agg}^G$ does not satisfies condition 2  of Definition \ref{def:contracted}.

Let $\x_0\in\M$ be an arbitrary point, and let $Agg$ be any aggregation function over $\M$. Next, we show that there exists a embedding $\bH$ of $G$ over $\M$ such that 
  $$d\left(\x_0,\ \bH'_{0,:}\right)\geq \max_{u_j\in\N(u_0)}\{d\left(\x_0, \bH_{j,:}\right)\},$$
where  $ \bH' = Agg^G(\bH)$ and $\{\bH_{j,:}\}_{u_j\in\N(u_i)}$  are not identical. If this is true, then $Agg^G$ does not satisfies condition 2  of Definition \ref{def:contracted}.

Now, we start finding such an embedding $\bH$. Since $\M$ is compact, there exists $\y_0\in\M$ such that

$$d(\x_0, \y_0) = \max_{\y\in\M} \{d(\x_0,\y)\}.$$

 Consider the map
 $f: \M^N\rightarrow \mathcal{M}$ 
defined by  
$$f(\bH_{j,:},\forall u_j\in\N(u_i)) = \bigg(Agg^G\Big( \bH\Big)\bigg)_{0,:}.$$ 
Then, $f$ is a subjective map since $Agg^G$.
By rank theorem \cite{Lee_2013}, the subset $f^{-1}(\y_0)\subseteq \M^N$ is a sub-manifold of dimension $d\cdot (N-1)$.

Moreover, the dimension of the sub-manifold of $\M^N$
$$\Q = \{\bH\in \M^N \mid \bH_{i,:} =\bH_{j,:}, \forall u_i, u_j \in\N(u_0)  \}$$
is $d\cdot(N- |\N(u_0)|+1) < d\cdot (N-1)$ as we assume $|\N(u_0)|>2$ and $d>0$. Therefore, there exists an embedding $\bH^*\in f^{-1}(\y_0) \setminus \Q$ such that $f(\bH^*) = \y_0$. 
Next, we show this $\bH^*$ is the desired embedding.

First, since $\bH^*\not\in\Q$, we know $\bH^*_{j,:}, u_j\in\N(u_0)$ are not identical.  Furthermore, let $\bH^{*'} = Agg^G(\bH^*)$,  we have 
$$\bH^{*'}_{0,:} = (Agg^G(\bH^*))_{0,:}  = f(\bH^*) = \y_0.$$
Since $d(\x_0, \y_0) = \max_{\y\in\M} \{d(\x_0,\y)\}$,  we have
  $$d\left(\x_0,\ \bH^{*'}_{0,:} \right) = d\left(\x_0,\ \y_0 \right)\geq \max_{u_j\in\N(u_0)}\{d\left(\x_0, \bH^*_{j,:}\right)\}.$$

Therefore, $\bH^*$ is the desired embedding. The theorem is proved.
\end{proof}

\subsection{Proof of Proposition \ref{prop:aggLi}}
\begin{proof}

For Cases 1 and 3, we have $\sum_{k=1}^m\bL_{i,k}=1$ and $\bL_{i, j}>0$ for all $1\leq i, j\leq m$. Then, $\mathit{Agg}^G$  satisfies Condition 1 of Definition \ref{def:contracted}. Moreover,  $\mathit{Agg}^G$ satisfies Condition 2 of Definition \ref{def:contracted} since for all $\x_1, \ldots, \x_m\in \R^n$ and positive weight $\{w_i\}_{i=1}^m$ s.t. $w_i>0, \sum_{i=1}^m w_i=1$, we have
$$d_{\R^n}(\x, \sum_{i=1}^m w_i\x_i) \leq \max(d_{\R^n}(\x, \x_1), \ldots, d_{\R^n}(\x, \x_m)),$$
where the equality holds if and only if $\x_1= \ldots =\x_m$. Therefore, $\mathit{Agg}^G$ is contracted.

Now, we prove the proposition holds for Case 2. Next, assume $\bL = (1-\lambda)\I+\lambda \tD^{ -\frac{1}{2}}\tA\tD^{-\frac{1}{2}}$ for some $\lambda\in(0, 1]$. 

Let $g: (\R^n)^N\rightarrow (\R^n)^N$ be the diffeomorphism map defined by $g(\textbf{H}) = \tD^{\frac{1}{2}}\textbf{H}$. Then, we have
$$(g^{-1}\circ \mathit{Agg}^G\circ g)(\textbf{H}) =\tD^{-\frac{1}{2}}\cdot\Big([(1-\lambda)\I+\lambda \tD^{ -\frac{1}{2}}\tA\tD^{-\frac{1}{2}}]\cdot (\tD^{\frac{1}{2}} \bH)\Big)= [(1-\lambda)\I+\lambda \tD^{-1}\tA]\cdot \bH.$$
Then, $g^{-1}\circ \mathit{Agg}^G\circ g$ is contracted by Case 1. Therefore, we have $\mathit{Agg}^G$ is equivalently contracted. 
 \end{proof}

\section{Experiment Setup}
\subsection{Dataset Description Statistics}
\label{sec:dataset}


In order to make a fair comparison with prior research, we conduct the node classification task using four homophilic datasets: Cora, Citeseer, Pubmed~\cite{yang2016revisiting}, and CoauthorCS~\cite{shchur2018pitfalls}, and five heterophilic datasets: Texas, Cornell, Wisconsin, Actor and Chameleon\cite{pei2020geom}. These datasets are widely used to investigate the problem of over-smoothing in GNNs. The statistics of the datasets are provided in Table~\ref{table:dataset}. 



In our experiments with Cora, Citeer, and Pubmed, we follow semi-supervised setup outlined in ~ \cite{kipf2016semi,zhou2020towards}, employing 20 nodes per class for training, 500 nodes for validation, and 1000 nodes for testing. For CoauthorCS, following~\cite{zhou2020towards}, we use 40 nodes per class for training, 150 nodes per class for validation, and allocating the remaining nodes for testing. For other datasets, we follow the~\cite{chen2020simple,pei2020geom}, randomly distribute nodes of each class into 60\%, 20\%, and 20\% for training, validation, and testing, respectively.

If a task needs many layers to perform well, using an over-smoothing alleviation approach would be beneficial. Therefore, we created a scenario that requires more layers to perform well. For each dataset, we create the scenario by removing node features in the validation and testing sets. We remove the input node embeddings of both validation and test sets, and replace them with zeros.

\begin{table}
\renewcommand*{\arraystretch}{1.1}
\footnotesize
	\setlength{\tabcolsep}{0.55mm}
	\center
    \caption{\small{Dataset statistics}} 
    \label{table:dataset}	
	\resizebox{.8\textwidth}{!}{
        \begin{tabular}{cccccc}
        \hline
Datasets  & \#Nodes  & \#Edges  & \#Features  & \#Classes  & Ave.Degree  \\ \hline
Cora  & 2708  & 5429  & 1433  & 7 &3.88  \\
Citeseer  & 3327  & 4732  & 3703  & 6  & 2.84\\
Pubmed  & 19717  & 44338  & 500  & 3 & 4.50 \\
CoauthorCS  & 18333  & 81894  & 6805  & 15  & 8.93\\ 
Texas & 183 & 325 & 1703 & 5 &3.38 \\ 
Wisconsin & 251 & 511 & 1703 & 5&5.45 \\ 
Cornell & 183 & 298 & 1703 & 5&3.22 \\ 
Actor & 7600 & 30019 & 932 & 5 &8.83\\ 
Chameleon  & 2277 & 36101 & 2325 & 5 & 13.79\\ \hline
                \end{tabular}
                }
\end{table}


\subsection{Hyperparameter Settings}
\label{sec:Hyperparameter}

For the proposed ACM, We use the following sets of hyperparapeters.
For Citeseer, Cora, and CoauthorCS: a learning rate of $5\times10^{-3}$, 1000 maximum epochs, 0.6 dropout rate, $5\times10^{-4}$ L2 regularization for Citeseer, $5\times10^{-5}$ L2 regularization for Cora and CoauthorCS. For Pubmed: a learning rate of $1\times10^{-2}$, 1000 maximum epochs, 0.6 dropout rate, $1\times10^{-3}$ L2 regularization. The hidden layer size is 16 units for four datasets. We also add this standard normalization operation BatchNorm after the transformation step on GAT and GCN.


\subsection{Baselines}
The following open source public available implementation of baselines which are used in this paper:
\begin{itemize}
    \item BatchNorm, PairNorm, DGN: \href{https://github.com/Kaixiong-Zhou/DGN/}{https://github.com/Kaixiong-Zhou/DGN/}
    \item DeCorr: \href{https://github.com/ChandlerBang/DeCorr}{https://github.com/ChandlerBang/DeCorr}
    \item DropEdge: \href{https://github.com/DropEdge/DropEdge}{https://github.com/DropEdge/DropEdge}
    \item DAGNN, GPRGNN, GCNII, APPNP:
    \href{https://github.com/VITA-Group/Deep_GCN_Benchmarking}{https://github.com/VITA-Group/Deep\_GCN\_Benchmarking}
\end{itemize}

\section{More Experiment Results}
\subsection{Figures}
\label{appendix:figures}
In Fig. \ref{fig:test_acc_SGC_MS0_all_appendix} illustrated the accuracy results of different over-smoothing alleviation methods on SGC, with increasing layers on all four datasets. 
Fig. \ref{fig:test_acc_SGC_MS1_all_appendix} shows the results under the scenario with missing features. 
\begin{figure*}
\vspace{-2ex}
\centering
\includegraphics[width=\linewidth]{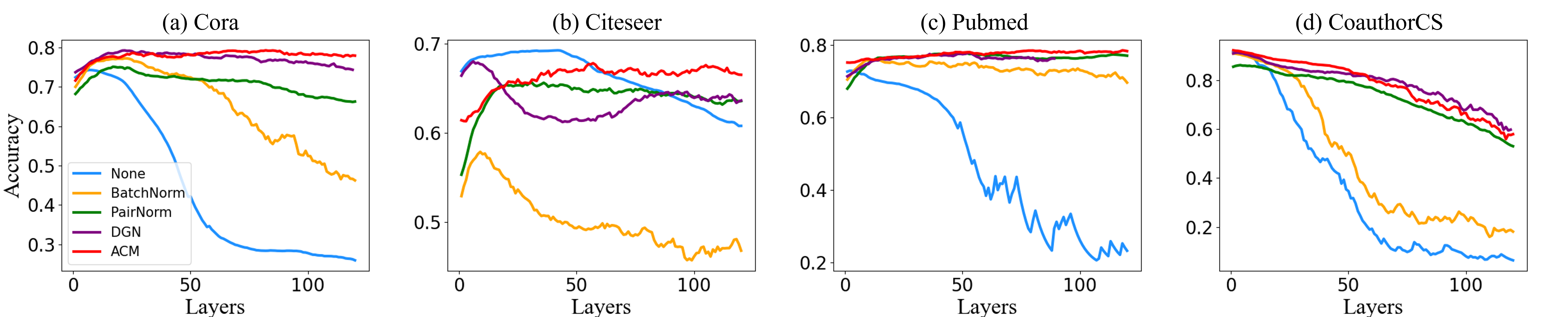}

\centering
\caption{Test accuracy of SGC with different methods and layers on all datasets, i.e., missrate 0.}
    \label{fig:test_acc_SGC_MS0_all_appendix}
    \vspace{-2ex}
\end{figure*}

\begin{figure*}
\vspace{-2ex}
\centering
\includegraphics[width=\linewidth]{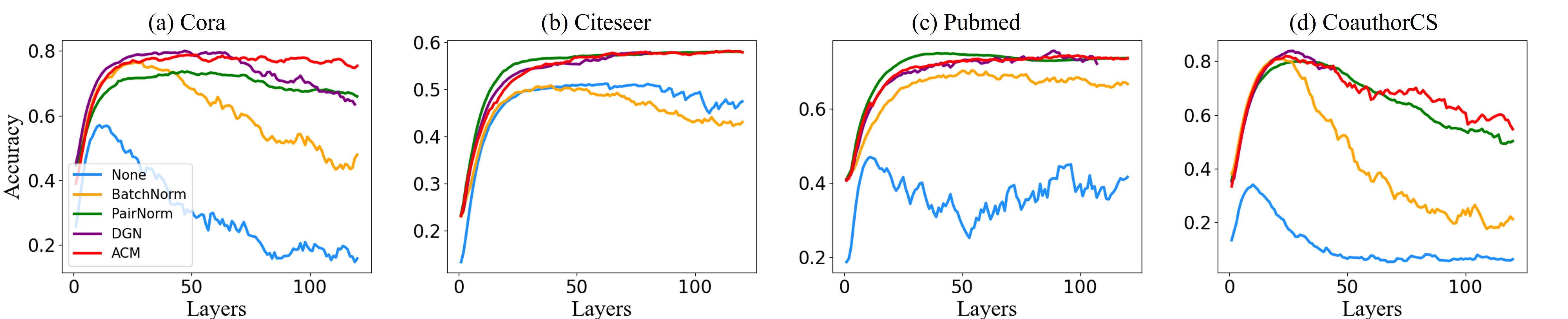}
\centering
\caption{Test accuracy of SGC with different methods and layers on all datasets with missing features, i.e., missrate 1.}
    \label{fig:test_acc_SGC_MS1_all_appendix}
    \vspace{-2ex}
\end{figure*}

\subsection{Tables}
\label{sec:Tablesappendix}
\begin{table}
\renewcommand*{\arraystretch}{1.1}
\footnotesize
\setlength{\tabcolsep}{1.5mm}
\centering
\caption{\small{Node classification accuracy (\%) on GCN, GAT.}}
\label{table:gatgcncomplete}
\resizebox{\textwidth}{!}{
\begin{tabular}{c|c|ccccc|ccccc|ccccc|ccccc}
\hline
\multicolumn{2}{c|}{Datasets} &     \multicolumn{5}{|c}{Cora} &      \multicolumn{5}{|c}{Citeseer} &      \multicolumn{5}{|c}{Pubmed} &   \multicolumn{5}{|c}{CoauthorCS}    \\ \hline
\textbf{Methods} & \textbf{\#Layers  } &  L2 & L15 & L30 & L45  &    L60 & L2 & L15 & L30 & L45  &    L60 &    L2 & L15 & L30 & L45  &    L60 &    L2 & L15 & L30 & L45  &    L60  \\ \hline
 \multirow{9}*{\textbf{GCN}}  &    None & 82.2  & 18.1  & 13.1  & 13.0  & 13.0  & 70.6  & 15.2  & 9.4  & 7.7  & 7.7  & 79.3  & 22.5  & 18.0  & 18.0  & 18.0  & 92.3  & 72.2  & 3.3  & 3.3  & 3.3   \\
 &   BN & 73.9  & 70.3  & 67.2  & 60.2  & 58.4  & 51.3  & 46.9  & 47.9  & 38.7  & 36.5  & 74.9  & 73.7  & 70.4  & 72.9  & 67.1  & 86.0  & 78.5  & 84.7  & 80.1  & 79.1   \\
 &   PN & 71.9  & 67.2  & 64.3  & 54.5  & 47.0  & 60.5  & 46.7  & 47.1  & 43.1  & 37.1  & 71.1  & 70.6  & 70.4  & 63.4  & 60.5  & 77.8  & 69.5  & 64.5  & 70.0  & 66.5   \\
 &   DropEdge    & \textbf{82.8} & 70.5  & 45.4  & 13.0  & 13.0  & 71.7  & 43.3  & 31.6  & 7.7  & 7.7  & 78.8  & 74.0  & 62.1  & 18.0  & 18.0  & 92.2  & 76.7  & 31.9  & 3.3  & 3.3   \\
 &   DGN & 82.0  & 75.2  & 73.2  & 67.8  & 52.6  & 69.5  & 53.1  & 52.6  & 45.8  & 40.5  & 79.5  & 76.1  & 76.9  & 73.4  & 72.8  & 92.3  & 83.7  & 84.4  & 83.7  & \textbf{82.1 }  \\
 &   DeCorr & 82.2  & \textbf{77.0} & 73.4  & 38.9  & 13.0  & \textbf{72.1} & \textbf{67.7} & \textbf{67.3} & 37.1  & 7.7  & \textbf{79.6} & \textbf{78.1} & \textbf{77.3} & 32.5  & 13.0  & \textbf{92.4} & \textbf{86.4} & 84.5  & 29.0  & 3.3   \\
 &   ACM & 77.2  & 74.5  & \textbf{73.5} & \textbf{71.6} & 69.5  & 67.3  & 56.4  & 55.1  & 56.5  & \textbf{53.5} & 78.9  & 75.7  & 75.7  & 74.3  & 74.4  & 91.6  & 83.9  & \textbf{85.5} & 84.6  & \textbf{82.4}  \\
 &   ACM* & 77.0  & 74.5  & 72.4  & \textbf{71.6} & \textbf{70.3} & 60.1  & 55.7  & 56.5  & \textbf{57.0} & 53.4  & 76.0  & 75.6  & 74.4  & \textbf{76.3} & \textbf{74.8} & 91.6  & 85.6  & 84.3  & \textbf{84.7} & 70.7   \\ \hline
 \multirow{9}*{\textbf{GAT}}  &    None & 80.9  & 16.8  & 13.0  & 13.0  & 13.0  & 70.2  & 22.6  & 7.7  & 7.7  & 7.7  & 77.8  & 37.5  & 18.0  & 18.0  & 18.0  & 91.5  & 6.0  & 3.3  & 3.3  & 3.3   \\
 &   BN & 77.8  & 33.1  & 25.0  & 21.6  & 16.2  & 61.5  & 28.0  & 21.4  & 21.1  & 18.1  & 76.2  & 56.2  & 46.6  & 45.3  & 29.4  & 89.4  & 77.7  & 16.7  & 4.2  & 2.6   \\
 &   PN & 74.4  & 49.6  & 30.2  & 28.8  & 19.3  & 62.0  & 41.4  & 33.3  & 30.6  & 27.3  & 72.4  & 68.8  & 58.2  & 58.8  & 58.1  & 85.9  & 53.1  & 48.1  & 30.4  & 26.6   \\
 &   DropEdge    & 81.5  & 66.3  & 51.0  & 13.0  & 13.0  & 69.8  & 52.6  & 36.1  & 7.7  & 7.7  & 77.4  & 72.3  & 64.7  & 18.0  & 18.0  & 91.2  & 75.0  & 52.1  & 3.3  & 3.3   \\
 &   DGN & 81.1  & 71.8  & 51.3  & 44.2  & 38.0  & 69.3  & 52.6  & 45.6  & 32.8  & 27.5  & 77.5  & 75.9  & 73.3  & 53.7  & 60.1  & \textbf{91.8} & 84.5  & 75.5  & 20.9  & 44.8   \\
 & DeCorr & \textbf{81.6} & \textbf{76.0} & 54.3  & 18.3  & 13.0  & \textbf{70.6} & \textbf{63.2} & 46.9  & 18.9  & 7.7  & \textbf{78.1} & \textbf{77.5} & 74.1  & 48.9  & 18.0  & 91.3  & 83.5  & 77.3  & 19.2  & 3.3   \\
 & ACM & 76.4  & 72.2  & 67.4  & 53.5  & 48.5  & 64.9  & 60.3  & 49.5  & 38.8  & 38.4  & 75.4  & 76.7  & 75.4  & 72.0  & 68.4  & 89.5  & \textbf{85.7} & 84.8  & 79.8  & 74.2   \\
 & ACM* & 78.3  & 74.0  & \textbf{71.4} & \textbf{64.4} & \textbf{61.3} & 64.7  & 59.2  & \textbf{56.2} & \textbf{48.3} & \textbf{47.2} & 77.4  & 74.7  & \textbf{76.7} & \textbf{75.0} & \textbf{70.8} & 90.9  & 85.4  & \textbf{85.0} & \textbf{82.6} & \textbf{76.3}  \\ \hline

\end{tabular}
}
\end{table}

\end{document}